%% file: acl_latex.tex
\documentclass[11pt]{article}

\usepackage[final]{acl}

\usepackage{times}
\usepackage{latexsym}

\usepackage[T1]{fontenc}

\usepackage[utf8]{inputenc}

\usepackage{microtype}

\usepackage{inconsolata}

\usepackage{graphicx}
\usepackage{booktabs}
\usepackage{url}
\usepackage{threeparttable}
\usepackage{multirow} 
\usepackage{dsfont}
\usepackage{makecell}
\usepackage{algorithm, algorithmic}
\usepackage{setspace}
\usepackage{subcaption}
\usepackage{listings}
\usepackage{xcolor}
\usepackage{pgffor}
\usepackage{wrapfig}
\usepackage{tikzsymbols}
\usepackage{wrapfig}
\usepackage{mathtools}
\usepackage{bm}
\usepackage{pifont}
\usepackage[ruled,vlined,algo2e]{algorithm2e}
\usepackage{amsmath}
\usepackage{amssymb}
\usepackage{amsfonts}
\usepackage{amsthm}
\SetKwComment{Comment}{/* }{ */}
\SetKwInput{Input}{Input}
\SetKwInput{Output}{Output}
\SetKw{KWD}{Keyword} 

\newtheorem{theorem}{Theorem}
\title{Weak Ties, Strong Signals: Efficient Training Data Detection in Diffusion LLMs via Independent Token Sampling}

\author{Hongyao Yu$^{1}$, \quad
Tianqu Zhuang$^{1}$, \quad
Ziyuan Xu$^{2}$, \quad
Hao Fang$^{1}$, \\
\textbf{Jiaxin Hong$^{2}$, \quad
Bin Chen$^{2}$\thanks{Corresponding author.}~, \quad
\quad Shu-Tao Xia$^{1}$}\\
$^{1}$ Tsinghua University 
$^{2}$ Harbin Institute of Technology, Shenzhen \\
\texttt{chrisqcwx@gmail.com}, \quad
\texttt{chenbin2021@hit.edu.cn}\\
}
\begin{document}
\maketitle

\input{custom_name}

\input{section/1_abstract}
\input{section/2_intro_v2}
\input{section/3_related_work}

\input{section/4_preliminary}

\input{section/5_method_ZTQ}

\input{section/6_experiment}

\input{section/7_conclusion}
\clearpage
\input{section/8_limitation}





\input{bibitems}
\appendix



\input{appendix/proof}
\input{table/main/main_dream}
\input{appendix/baseline}

\input{appendix/details}

\input{appendix/add_exp}
\input{appendix/ai}
\end{document}

%% file: custom_name.tex
\def\ie{\textit{i.e.}}
\def\eg{\textit{e.g.}}
\def\ourmethod{TI_MIA}

%% file: section/1_abstract.tex
\begin{abstract}

Diffusion large language models (dLLMs) offer a compelling alternative to autoregressive models, yet they may expose sensitive training data during denoising.
Detecting such usage is challenging because dLLMs lack the efficient one-pass probability decomposition of causal architectures.
Existing methods rely on random masking to obtain tractable token-wise detection signals under limited query budgets, but fail to control dependencies among masked tokens.
We demonstrate that this token-wise approximation introduces a non-negative structural estimation error, which is theoretically characterized by the cumulative conditional mutual information (CMI) among masked tokens and can obscure subtle memorization signals.
This insight suggests that reliable detection requires masked token sets with weak internal dependency. 
To avoid the prohibitive cost of directly estimating CMI over token combinations, we propose \textit{Independent Token Sampling} (ITS), a query-efficient framework that uses an attention-derived pairwise dependency proxy to approximate the CMI-aware selection criterion. 
ITS further incorporates a diversity-promoting strategy to improve token coverage across sampling rounds, yielding aggregated token-wise signals that are less affected by dependency-induced approximation error. 
Experiments on multiple datasets show that ITS consistently outperforms state-of-the-art baselines across different models and datasets, achieving an AUC improvement of 0.18 on the ArXiv dataset while maintaining strong performance under limited query budgets.
The code is available at \url{https://github.com/Chrisqcwx/DLLM-MIA}.

\end{abstract}

%% file: section/2_intro_v2.tex
\section{Introduction}

\begin{figure}
    \centering
    \includegraphics[width=\linewidth]{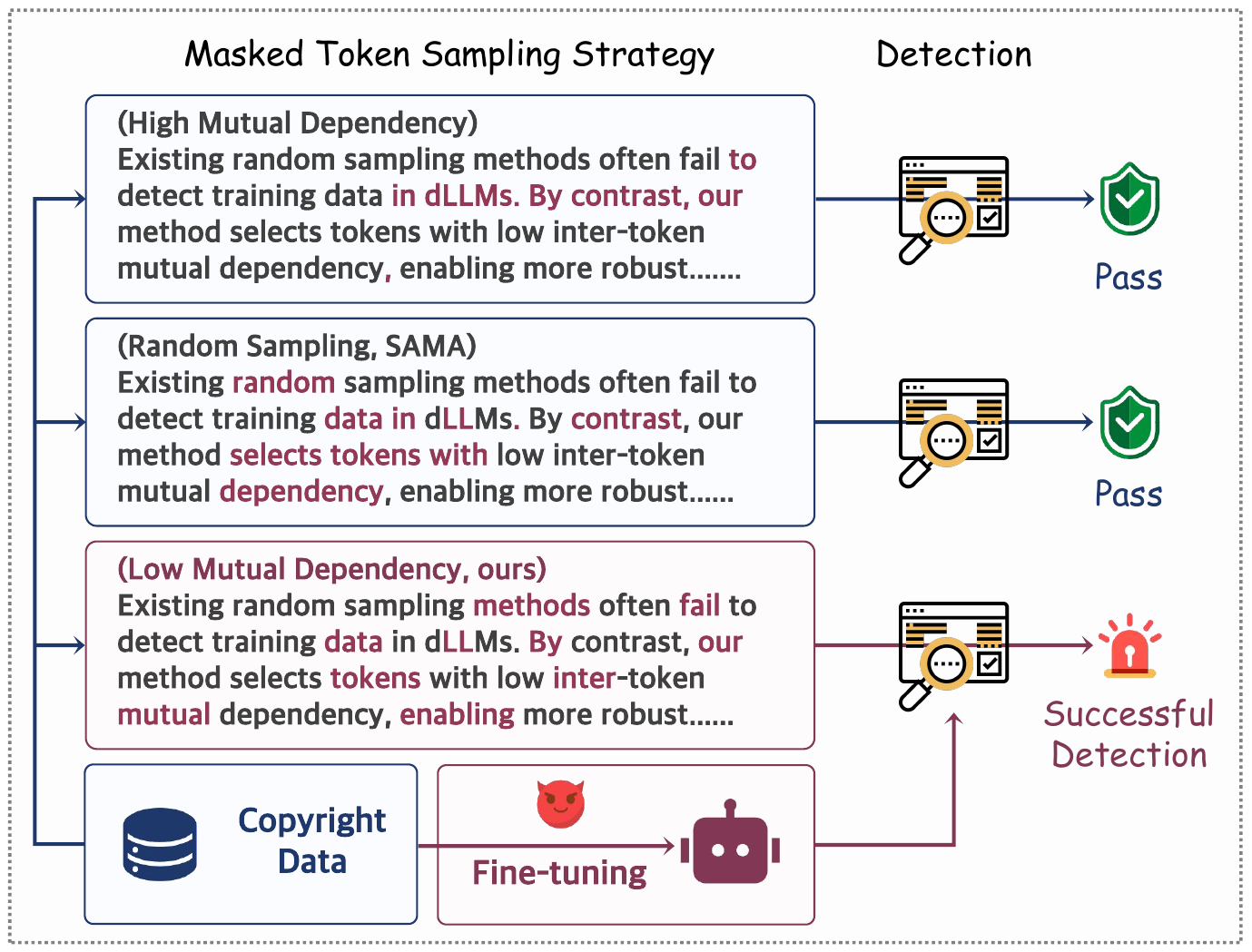}
    \caption{Illustration of token sampling strategies. Existing random sampling methods often fail to detect training data in dLLMs. By contrast, our method selects weakly dependent token subsets, enabling more robust training data detection.}
    \label{fig:intro}
\end{figure}

{
\def\toywidth{0.27\textwidth}
\begin{figure*}[!tbp]
    \centering
    \begin{subfigure}[b]{\toywidth}
        \centering
        \includegraphics[width=\textwidth]{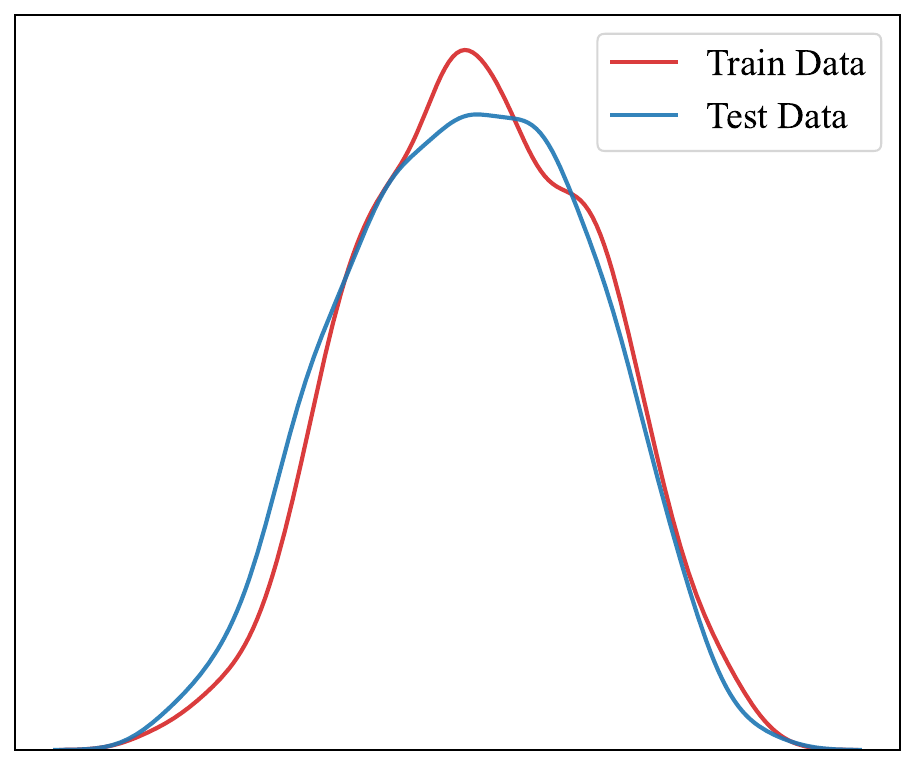}
        \caption{High Dependency}
        \label{fig:low pcmi}
    \end{subfigure}
    \begin{subfigure}[b]{\toywidth}
        \centering
        \includegraphics[width=\textwidth]{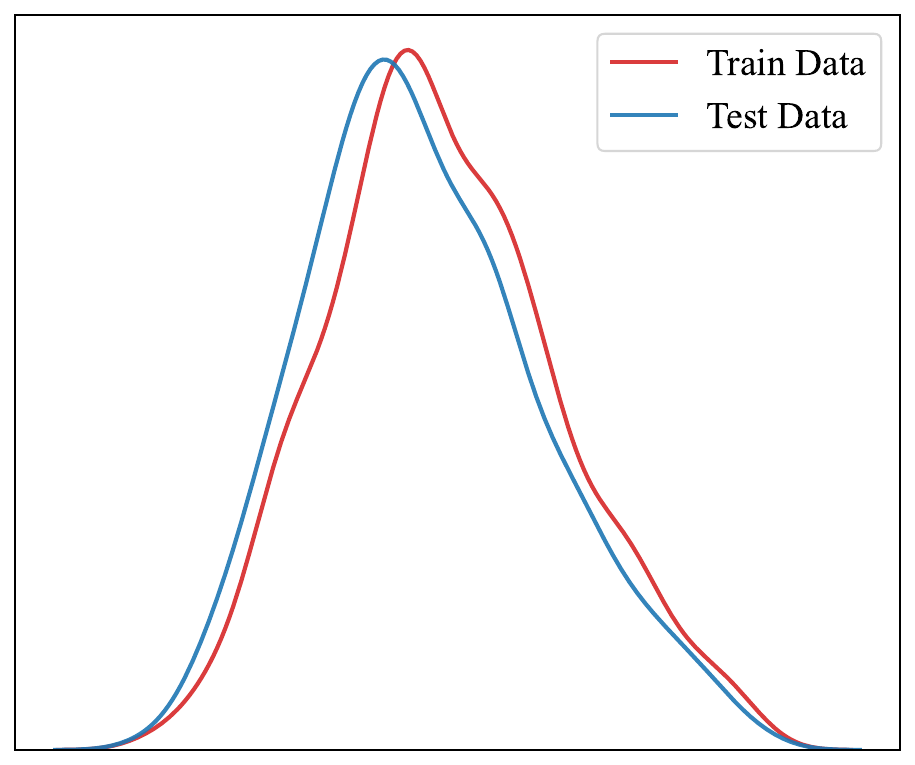}
        \caption{Random Sampling}
        \label{fig:mid pcmi}
    \end{subfigure}
    \begin{subfigure}[b]{\toywidth}
        \centering
        \includegraphics[width=\textwidth]{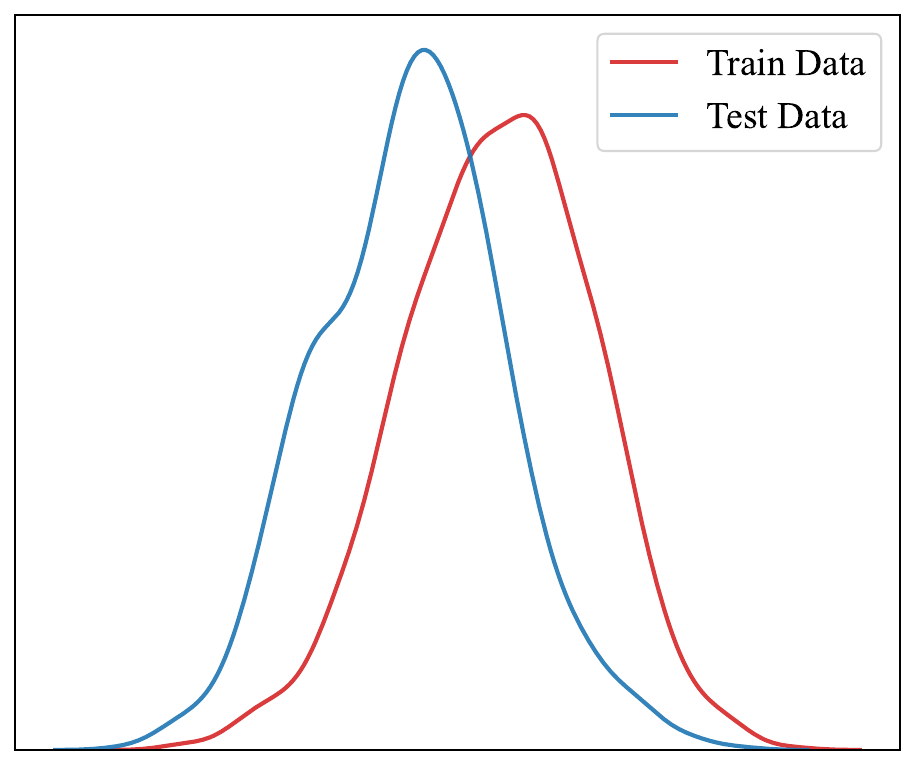}
        \caption{Low Dependency}
        \label{fig:high pcmi}
    \end{subfigure}
\caption{
Distribution of sample-level detection scores under different token sampling strategies on LLaDA fine-tuned with ArXiv dataset. 
Each score is computed for one sample as the proportion of selected tokens where the target model assigns a higher probability than the reference model, and the distribution is obtained by aggregating scores over many samples. 
Compared with (a) high-dependency sampling and (b) random sampling, our (c) low-dependency strategy yields a clearer separation between training data (red) and test data (blue).
}
    \label{fig:motivation}
\end{figure*}
}

Are all masked token combinations equally reliable for exposing memorization on diffusion large language models?
Is the dependency among them merely a sampling artifact, or a decisive factor in whether the probability estimate remains faithful?
We argue the latter: for diffusion Large Language Models (dLLMs), memorization is more reliably exposed by token combinations that provide weakly dependent evidence.
This question is central to training data detection for modern diffusion language models.
Training data detection \citep{mink,icas} is becoming increasingly important for modern dLLMs, as their growing capabilities in parallel generation and bidirectional reasoning raise privacy \citep{ifgmi,ssd,loft} and copyright \citep{xu2026bypassing,gausstrap,p2f} concerns about memorized fine-tuning data. However, detecting such memorization is fundamentally more challenging than in autoregressive LLMs (AR-LLMs). AR-LLMs admit a causal factorization of sequence likelihood into token-level conditional probabilities, which can be efficiently computed in a single forward pass \citep{mink,minkpp,dfmia}. In contrast, dLLMs estimate token probabilities under masked contexts, making the resulting detection signal depend on the queried masking configurations.

We view this challenge through the lens of probability approximation. Ideally, training data detection should rely on a faithful estimate of the target sequence probability, but exact marginalization over all masking patterns is computationally infeasible because the number of possible configurations grows exponentially with sequence length. Practical detection must therefore approximate the desired joint probability using only tractable token-wise probabilities. 
The key question is therefore: \textit{How can we obtain a faithful detection signal from tractable token-wise probabilities without exhaustive computation?}

A common solution is random token sampling \citep{sama}, where a subset of tokens is masked and the model is scored by aggregating token-wise reconstruction signals. 
This makes probability estimation computationally tractable, but leaves the dependency structure of the sampled token set uncontrolled. 
The resulting surrogate implicitly treats sampled tokens as if their probability contributions can be accumulated independently, introducing a structural estimation error.
As illustrated in Figure~\ref{fig:intro}, this error can smear discriminative training signals and reduce detection reliability.
We further prove that this structural error is characterized by the cumulative conditional mutual information (CMI) among sampled tokens.

This analysis suggests that the failure of random sampling is not merely due to insufficient samples, but to a dependency mismatch in probability approximation. 
A principled solution should therefore prefer token subsets with weak internal dependency.
To this end, we propose \textit{Independent Token Sampling} (ITS), an attention-guided framework for training data detection in dLLMs. 
Since directly estimating CMI for all candidate token combinations is computationally expensive, ITS leverages the model's internal attention structure to build a pairwise dependency proxy and selects token subsets with low estimated dependency.
To improve the use of the sampling budget, we further introduce a diversity-promoting penalty, which encourages complementary token coverage across different sampling rounds.

Empirical evaluations demonstrate the effectiveness of ITS in the privacy and copyright protection \citep{mink, fang2024privacy}. 
Our method consistently outperforms state-of-the-art (SOTA) baselines across multiple datasets and remains robust under challenging scenarios such as sparse model updates with LoRA and partial observations. 
These results show that modeling dependency structure is essential for faithful probability approximation in dLLMs.
Our contributions are summarized as follows:
\begin{itemize}

    \item \textbf{Theoretical Insight:} We identify the structural estimation error introduced by approximating joint probability with token-wise probability products. This error is characterized by the cumulative conditional dependency among masked tokens.

    \item \textbf{Independent Token Sampling:} We propose ITS, an attention-guided token sampling framework that selects weakly dependent token subsets. A diversity-promoting penalty is further introduced to improve token coverage across sampling rounds.

    \item \textbf{Empirical Validation:} We demonstrate that ITS significantly outperforms state-of-the-art baselines across different scenarios.
\end{itemize}

%% file: section/3_related_work.tex
\section{Related Work}

\subsection{Diffusion Large Language Models}


Diffusion large language models have emerged as a scalable alternative to the autoregressive paradigm. Recent breakthroughs such as LLaDA \citep{llada} and Dream \citep{dream} demonstrate that dLLMs can match similarly sized autoregressive models in performance while excelling in parallel decoding \citep{paralleldecoding} and mitigating the reversal curse \citep{reversal}. Architectural scaling continues to advance through Mixture-of-Experts \citep{lladamoe} and extremely large-scale models like LLaDA 2.0 \citep{llada2}, which extends dLLMs to the 100B parameter regime. These efforts collectively transition dLLMs from theoretical interest to practically viable alternatives for large-scale deployment \citep{seed}.

\subsection{Training Data Detection on (d)LLMs}

Training data detection aims to identify whether a specific sample was included in a model's training set, a task crucial for copyright auditing. In the context of autoregressive (AR) models, early methods utilized loss-based signals \citep{lossattack} or likelihood ratios with reference models \citep{ratio}. Recent advances focus on token-level statistics like Min-k\% \citep{mink, minkpp}, input perturbations \citep{neighbor}, or contrastive conditional likelihoods \citep{conrecall,dfmia}. ICAS \citep{icas} extends it to AR image generative models.

However, these methods struggle with the unique inference mechanisms of dLLMs, where training signals are highly dependent on specific masking patterns. To address this, SAMA \citep{sama} introduces a progressive mask-density sampling and sign-based aggregation strategy to mitigate signal sparsity. While SAMA improves detection for dLLMs, the base question for probability approximation is not resolved.


%% file: section/4_preliminary.tex





\section{Preliminaries}

\subsection{Diffusion Large Language Models}

Diffusion large language models (dLLMs) generate text through masked denoising rather than left-to-right prediction. Given an input sequence ${x}$ of length $|x|$, a subset of positions $S\subseteq [|x|]$, where $[|x|]=\{1,\dots,|x|\}$, is sampled and replaced with a special \texttt{[MASK]} token. We use ${x}_S$ to denote the original tokens at masked positions and ${x}_{[|x|]\setminus S}$ to denote the partially observed sequence.
The core module of dLLMs is a mask predictor $\mathcal{P}(\cdot\mid {x}_{[|x|]\setminus S})$, which predicts all masked tokens in parallel. 
%
At inference time, generation starts from a fully or partially masked sequence, and the model iteratively predicts and reveals tokens until a complete sequence is obtained. Since masked positions can be predicted in parallel, dLLMs avoid the strict left-to-right constraint of autoregressive models and are naturally suited for bidirectional reasoning, infilling, and controllable generation.

\subsection{Problem Formulation}

We study the problem of \textit{training data detection} \citep{mink} for diffusion large language models, focusing on the fine-tuning stage \citep{sama}. Specifically, given a pre-trained reference model $\mathcal{P}_{\text{ref}}$ and a training dataset $\mathcal{D}$, the model is fine-tuned to obtain $\mathcal{P}_{\theta}$. 
The goal of an auditor is to determine whether a specific data sample $x$ is used to train the model, \ie, $x\in \mathcal{D}$ or $x\notin \mathcal{D}$.



Adopting the grey-box framework from \citet{sama}, we assume the auditor possesses the capability to query the target model with arbitrary, partially masked sequences. Under this setting, the auditor observes the resulting output distributions (e.g., logits) but lacks access to the model's internal weights or gradients. Furthermore, we grant the auditor full access to publicly available, pre-trained reference models to facilitate the auditing process.

%% file: section/5_method_ZTQ.tex
\begin{figure*}[t]
    \centering
    \includegraphics[width=.95\linewidth]{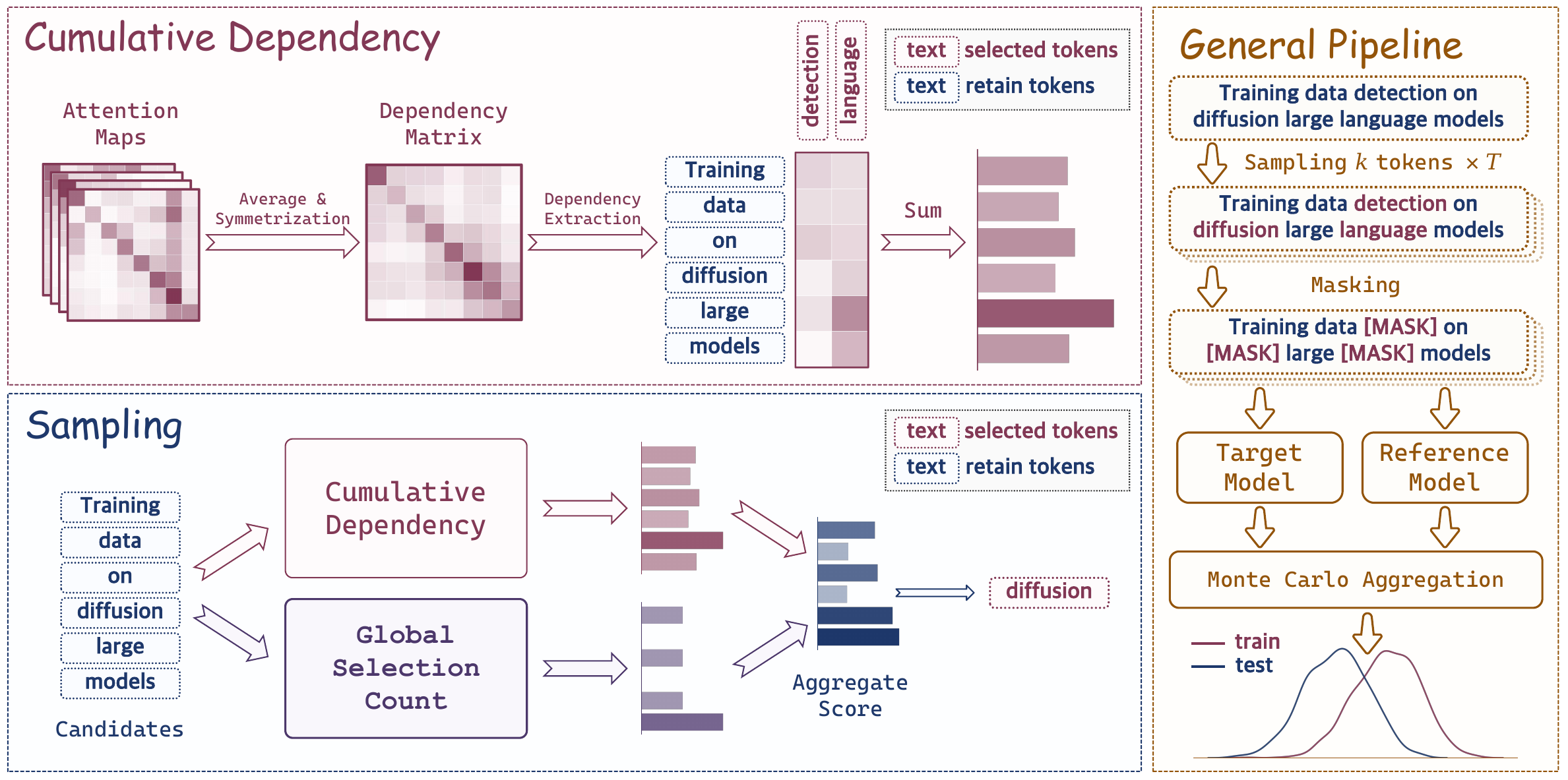}
    \caption{Overview of our proposed ITS detection method}
    \label{fig:main}
\end{figure*}

\section{Method}

\subsection{Full-Sequence Signal for Detection}

Let $\mathcal{P}_\theta(x)$ and $\mathcal{P}_\text{ref}(x)$ be the generative distributions of the target and reference dLLMs, respectively. 
A common observation in detection is that a fine-tuned model tends to assign higher probability to its training samples than a comparable reference model \citep{dfmia}:
\begin{equation}
x\in\mathcal{D}\implies\mathcal{P}_\theta(x)>\mathcal{P}_\text{ref}(x).
\end{equation}
Thus a natural signal for detection is
\begin{equation}
\Delta(x)\coloneqq\frac{\mathcal{P}_\theta(x)}{\mathcal{P}_\text{ref}(x)}.
\end{equation}
The higher $\Delta(x)$ indicates the greater confidence of $x\in\mathcal{D}$.

However, unlike AR-LLMs \citep{mink,ratio}, dLLMs do not provide an explicit tractable factorization of the full-sequence probability $\mathcal{P}(x)$. 
A common approximation is to replace the intractable joint probability with a product of token-wise probabilities \citep{llada,sama}:
\begin{equation}
\hat{\mathcal{P}}(x)\coloneqq\prod_{i=1}^{\lvert x\rvert}\mathcal{P}(x_i).
\end{equation}
But $\hat{\mathcal{P}}(x)$ is inaccurate due to the dependence among $x_1,x_2,\dots,x_{\lvert x\rvert}$. 

KL divergence measures the distance between two probability distributions.
With this metric, we define the expected error of $\hat{\mathcal{P}}(x)$ as
\begin{equation}
\text{error}\Big[\hat{\mathcal{P}}(x)\Big]\coloneqq\mathcal{KL}\Big(\mathcal{P}(x)\Big\Vert\hat{\mathcal{P}}(x)\Big).
\end{equation}

This error can be determined by the mutual information $\mathcal{I}$ among all tokens, which is revealed by our Theorem~\ref{thm:定理1}.

\begin{theorem}[Full-Sequence Error Decomposition]
\label{thm:定理1}
\begin{equation}
\textnormal{error}\Big[\hat{\mathcal{P}}(x)\Big]=\sum_{i=2}^{\lvert x\rvert}\mathcal{I}\big(X_{<i};X_{i}\big).
\end{equation}
\end{theorem}
The proof is provided in Appendix~\ref{sec:证明}.
Mutual information $\mathcal{I}$ measures the dependency between two random variables.
In practice, the total $\mathcal{I}$ among all tokens is substantial, leading to an intolerable error when estimating $\mathcal{P}(x)$. This error renders $\Delta(x)$ unreliable for the detection.
This motivates the use of selected token subsets, which can reduce the dependency-induced approximation error while preserving useful detection signals.

\subsection{Subset Signal for Error Reduction}


Motivated by the above analysis, we turn from the full sequence to selected token subsets. 
Let $S\subseteq[\lvert x\rvert]$ be the set of masked positions, and let $c\coloneqq x_{[\lvert x\rvert]\setminus S}$ denote the unmasked context. 
We define the selected-token signal as
\begin{equation}
\Delta(x_S|c)\coloneqq\frac{\mathcal{P}_\theta(x_S|c)}{\mathcal{P}_\text{ref}(x_S|c)},\quad S\subseteq[\lvert x\rvert].
\end{equation}

By masking the positions in $S$ and querying the dLLM with the visible context $c$, we obtain token-wise conditional probabilities $\mathcal{P}(x_{S_i}|c)$ for $i=1,2,\dots,\lvert S\rvert$.
Similarly, the conditional probability $\mathcal{P}(x_S|c)$ is estimated by:
\begin{equation}
\hat{\mathcal{P}}(x_S|c)\coloneqq\prod_{i=1}^{\lvert S\rvert}\mathcal{P}(x_{S_i}|c).
\end{equation}

Similar to $\hat{\mathcal{P}}(x)$, we define the expected error of the conditional probability $\hat{\mathcal{P}}(x_S|c)$ as:
\begin{equation}
\text{error}\Big[\hat{\mathcal{P}}(x_S|c)\Big]\coloneqq\mathcal{KL}\Big(\mathcal{P}(x_S|c)\Big\Vert\hat{\mathcal{P}}(x_S|c)\Big).
\end{equation}
This error can be determined by the conditional mutual information (CMI) among the selected tokens, which is revealed by our Theorem~\ref{thm:定理2}.
\begin{theorem}[Subset Error Decomposition]
\label{thm:定理2}
\begin{equation}
\textnormal{error}\Big[\hat{\mathcal{P}}(x_S|c)\Big]=\sum_{i=2}^{\lvert S\rvert}\mathcal{I}\big(X_{S_{<i}};X_{S_{i}}\big|c\big).
\end{equation}
\end{theorem}

Theorem~\ref{thm:定理2} reveals that the selected-token signal becomes more reliable when the masked tokens are weakly dependent given the context. 
Therefore, we formulate a greedy dependency-aware selection principle. 
Each token is selected by
\begin{equation}
S_i\gets\mathop{\arg\min}_{j\in[\lvert x\rvert]\setminus S_{<i}}\mathcal{I}\big(X_{S_{<i}};X_j\big|c'\big),
\end{equation}
where $c'\coloneqq x_{[\lvert x\rvert]\setminus(S_{<i}\cup\{j\})}$ denotes the visible context when the previously selected tokens and the candidate token are masked, for $i=2,3,\dots,\lvert S\rvert$. 
This criterion selects the token with the weakest dependency on the previously selected tokens.

\begin{algorithm}[t]
\SetAlgoLined
\LinesNumbered
\KwIn{Dependency matrix $\mathbf{A}\in\mathbb{R}^{L\times L}$, sampling rounds $T$, selection sizes $\{k_1,\dots,k_T\}$, diversity penalty $\lambda$}
\KwOut{Selected token sets ${S}^{1},\dots,{S}^{T}$}

Initialize global selection counts $\mathbf{C}\leftarrow \mathbf{0}$\; 

\For{$t=1$ \KwTo $T$}{
    Select seed token $v^* \leftarrow \arg\min_{v\in[L]} C_v$\;
    
    Initialize ${S}^{t} \leftarrow \{v^*\}$\;
    
    Update $C_{v^*}\leftarrow C_{v^*}+1$\;
    
    Initialize cumulative dependency $\mathbf{h}\leftarrow \mathbf{A}_{v^*,:}$\;
    
    \While{$|{S}^{t}| < k_t$}{
        Compute candidate scores:\\
        $u_v \leftarrow h_v + \lambda C_v,\quad \forall v\notin {S}^{(t)}$\;
        
        Select $v_{\mathrm{next}} \leftarrow \arg\min_{v\notin {S}^{(t)}} u_v$\;
        
        Update ${S}^{t} \leftarrow {S}^{t} \cup \{v_{\mathrm{next}}\}$\;
        
        Update $C_{v_{\mathrm{next}}}\leftarrow C_{v_{\mathrm{next}}}+1$\;
        
        Update $\mathbf{h}\leftarrow \mathbf{h}+\mathbf{A}_{v_{\mathrm{next}},:}$\;
    }
}
\Return{${S}^{1},\dots,{S}^{T}$}\;
\caption{Independent Token Sampling (ITS)}
\label{alg:its}
\end{algorithm}

\subsection{Independent Token Sampling}

Motivated by the greedy CMI-aware principle for constructing masked token sets described above, we introduce Independent Token Sampling (ITS).
However, directly estimating the conditional mutual information for all candidates is computationally expensive. 
To make this principle practical, we leverage the internal attention structure of dLLMs to build a pairwise dependency proxy for the CMI-based selection criterion. 
While this attention-derived proxy does not directly measure conditional mutual information, it provides an efficient model-internal estimate of relative dependency between token positions \citep{wen2022revisiting}. 
Therefore, ITS selects token combinations with low estimated dependency to reduce the error characterized by Theorem~\ref{thm:定理2}.

Specifically, let $\mathbf{A}_{l,h}$ be the attention map of the $l$-th layer and $h$-th head,
\begin{align}
\overline{\mathbf{A}}&\coloneqq\mathop{\mathbf{mean}}_{l,h}\mathbf{A}_{l,h},\\
\mathbf{A}&\coloneqq\overline{\mathbf{A}}+\overline{\mathbf{A}}^T.
\end{align}

We obtain $\mathbf{A}_{l,h}$ from the available reference model.
$\mathbf{A}$ represents pairwise token dependency, where a lower value $A_{ij}$ indicates weaker estimated dependency between token positions $i$ and $j$.

Given $\mathbf{A}$, ITS constructs masked sets through the greedy procedure in Algorithm~\ref{alg:its}. 
For each sampling round, the algorithm starts from a seed token and then progressively adds the token with the lowest cumulative dependency to the already selected set. 
For a candidate token $v$, this dependency score is computed by accumulating its dependency with all previously selected tokens. 
This encourages the masked set to contain weakly connected tokens, making the token-wise signal approximation less affected by internal dependency.

To avoid repeatedly selecting the same isolated positions across different rounds, ITS further introduces a diversity-promoting penalty $\lambda C_v$ into the candidate score, where $C_v$ records the number of selections for position $v$. 
This penalty encourages complementary low-dependency masked sets while preserving the dependency-reduction principle.

Overall, ITS implements the CMI-aware selection principle through an attention-derived dependency proxy. 
By reducing dependency within each masked set and improving coverage diversity across rounds, ITS provides more reliable token-wise signals for training data detection in dLLMs.

\subsection{General Pipeline}

Following the hierarchical masking framework proposed in \citet{sama}, we compute the score by querying the target model for \(T\) rounds. At each query step \(t\), a masking ratio \(\rho_t\) is specified, with values linearly spaced between \(\rho_{\min}\) and \(\rho_{\max}\), and the masked subset \({S}^{(t)}\) is selected using ITS. The point-wise log-likelihood ratio is then defined as follows:
\begin{equation}
    \phi(i,S^t)\coloneqq\frac{\mathcal{P}_{\theta}(x_i | {x}_{[|x|]\setminus{S}^{t}})}{{\mathcal{P}_\text{ref}(x_i | {x}_{[|x|]\setminus{S}^{t}}})}.
\end{equation}

Within this framework, we further employ a Monte Carlo aggregation strategy by randomly drawing $N$ subsets $\{{S}^{t}_{(j)}\}_{j=1}^N$ of size $M$ from ${S}^{t}$, and computing the proportion of subsets with a positive cumulative ratio \citep{sama}:

\begin{equation}
    \phi(S^t) = \frac{1}{N} \sum_{j=1}^N \mathds{1}\Big[\prod_{i \in {S}^{t}_{(j)}}\phi(i,S^t)>1\Big].
\end{equation}

The final score is the average proportion across all $T$ rounds: $\text{Score}({x}) = \frac{1}{T} \sum_{t=1}^T \phi({S}^{t})$. The higher score indicates that the sample ${x}$ is more likely to be contained in the training dataset. Unlike standard random masking, our ITS ensures that tokens within ${S}^{t}$ are sampled from non-redundant regions, maximizing the signal for detection.

%% file: section/6_experiment.tex
\input{table/main/main_llada}
\section{Experiment}

\subsection{Experimental Setup}

\paragraph{Target Models and Datasets.} Following \citet{sama}, we evaluate our method on two types of diffusion large language models: LLaDA \citep{llada} and Dream \citep{dream}. The former is trained from scratch, while the latter is adapted from autoregressive language models (AR-LLMs). We conduct experiments on six datasets, including AG News \citep{agnews} and five domains from the MIMIR benchmark \citep{mimir}, namely ArXiv, Pile-CC, Wikipedia (en), PubMed Central, and GitHub. We follow \citet{sama} for the construction of the training and test splits. The training details are provided in Appendix \ref{appx:exp details training}.

\paragraph{Evaluation Metrics.} 
We evaluate the effectiveness of the detection method using standard metrics \citep{carlini2022membership}, including the area under the ROC curve (AUC) and the true positive rate (TPR) at $10\%$ and $1\%$ false positive rates (FPR), denoted as T@$10\%$ and T@$1\%$, respectively. Full numerical results will be provided in Appendix \ref{appx:additional results}.

\paragraph{Baselines.} 
We compare our method with several representative baselines, including Loss \citep{lossattack}, Min-k\% \citep{mink}, Min-k\%++ \citep{minkpp}, Ratio \citep{ratio}, DF-MIA \citep{dfmia}, and SAMA \citep{sama}. Among them, only SAMA is specifically designed for diffusion large language models, while the other baselines were originally proposed for AR-LLMs and are adapted following the protocol in \citet{sama}.


\paragraph{Implementation Details.}
We follow the implementation of SAMA~\citep{sama} for all experimental settings, except for the sampling budget. Specifically, we uses a small budget of $T=4$ throughout the experiments. Additional implementation details and hyperparameters are provided in Appendix~\ref{appx:exp details inference}.

\subsection{Comparison to Baselines}


Table~\ref{tab:main performance} reports the detection performance of ITS and existing baselines on six datasets using LLaDA. Overall, ITS achieves the strongest performance across datasets and metrics, showing that selecting weakly dependent masked token sets provides a more reliable signal than random masking or autoregressive-style scoring heuristics.


ITS obtains an average AUC of approximately $0.87$, improving over the strongest baseline SAMA ($0.77$) by about $13\%$ relatively. The improvement is especially clear compared with traditional detection method for AR-LLMs such as Loss, Zlib, Min-k\%, Min-k\%++, DF-MIA, and Ratio, which mostly remain close to random guessing. This suggests that methods designed for autoregressive do not directly transfer to diffusion paradigm.



Across datasets, ITS remains consistently effective under different levels of detection difficulty. On high-signal datasets such as Pile CC and Wikipedia, ITS achieves near-saturated AUC scores of $0.99$ and $0.95$, respectively, while also providing large gains at strict FPR thresholds. On more challenging datasets such as ArXiv, PubMed Central, and AG News, where most baselines are substantially weaker, ITS still maintains clear AUC improvements. This robustness supports our hypothesis that controlling dependency within the masked token set reduces likelihood smearing and preserves subtle memorization signals.


We further evaluate ITS on the Dream model \citep{dream}, with results reported in Appendix~\ref{appx:dream result}. ITS maintains clear advantages over all baselines and achieves similar AUC improvements to those observed on LLaDA. This suggests that the effectiveness of ITS is not tied to a specific dLLM architecture, but generalizes across different diffusion language models.

\subsection{Detection on LoRA: Localization of Training Signals}

To investigate the architectural localization of training signatures, we evaluate ITS on models fine-tuned via Low-Rank Adaptation (LoRA). We vary the rank $r \in \{1, \dots, 128\}$ and compare the detectability of training data when targeting either MLP or Attention layers. As illustrated in Figure \ref{fig:lora}, two primary insights emerge:

 \textbf{Superiority of ITS.} Across all ranks and modules, ITS consistently outperforms the SAMA baseline. For instance, when targeting the MLP at $r=2$, ITS achieves an AUC of 0.85, significantly higher than SAMA's 0.69. This indicates that ITS maintains robust detection performance even with sparse model updates, outperforming baselines in regimes with limited training evidence.

\begin{figure}[t]
    \centering
    \includegraphics[width=0.85\linewidth]{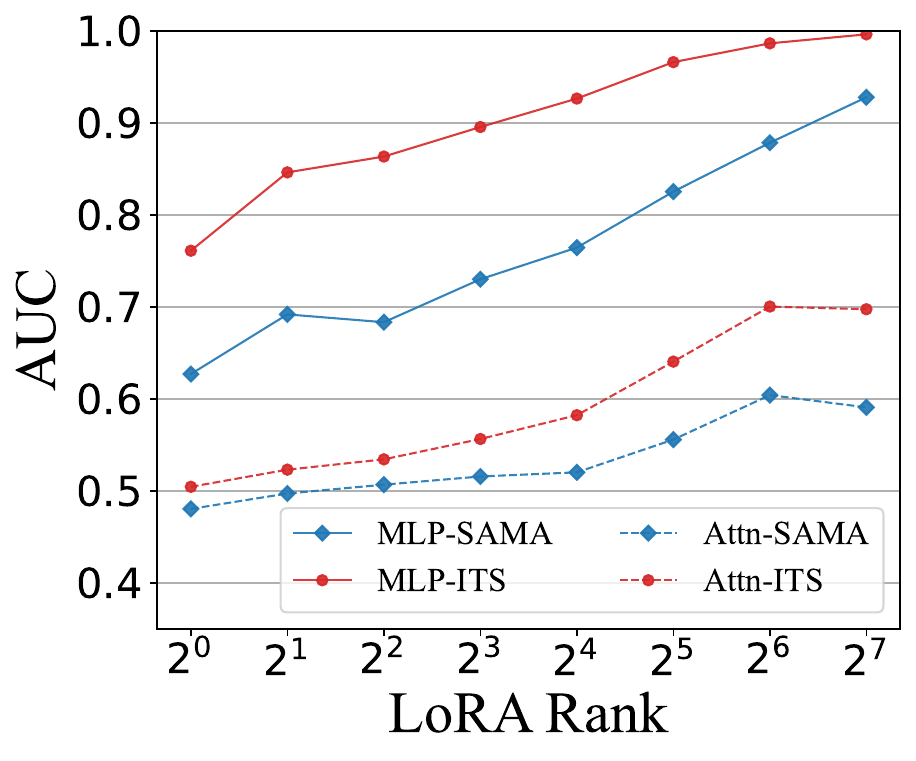}
    \caption{Training data detection performance on LoRA-tuned models across different ranks $r$ and modules.}
    \label{fig:lora}
\end{figure}

\textbf{MLP as the Knowledge Repository.} The detection is substantially more effective when targeting MLP layers compared to attention blocks. At $r=32$, the AUC for ITS targeting MLPs reaches a near-perfect 0.97, while its performance on Attention layers is only 0.64. This disparity aligns with interpretability research characterizing MLPs as key-value memories for training data \citep{mlp1, mlp2}, whereas attention layers primarily manage contextual routing.
These results confirm that copyrighted information is predominantly internalized within MLP weights. Furthermore, they demonstrate that ITS provides a highly sensitive mechanism for identifying these localized training signals, even in constrained low-rank adaptation scenarios.

\begin{figure}[!h]
    \centering
    \includegraphics[width=0.8\linewidth]{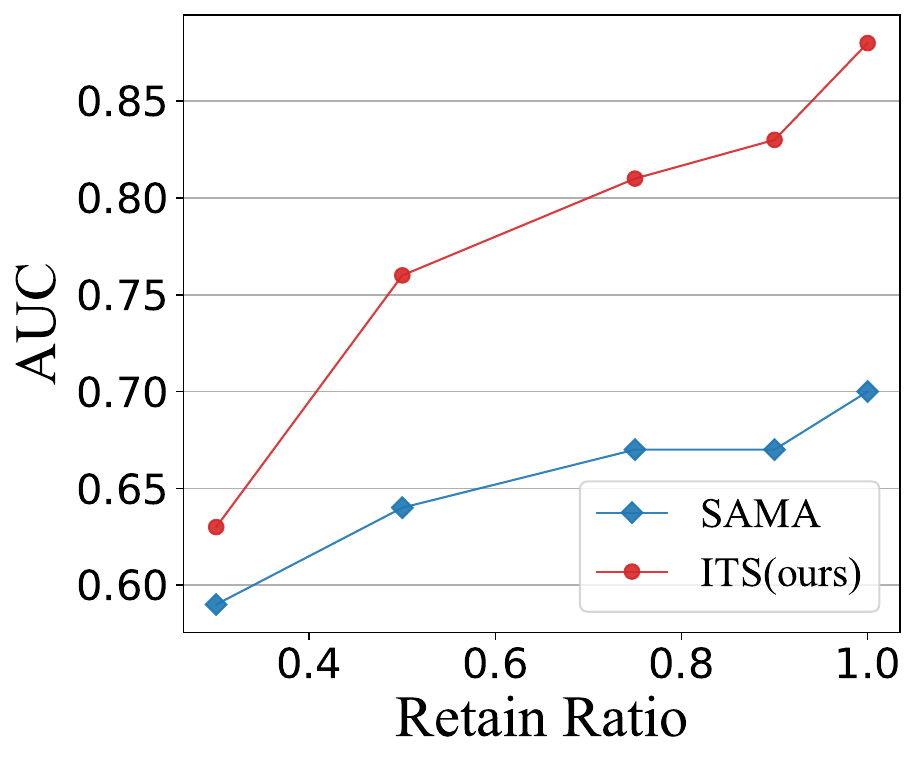}
    \caption{Detection performance across varying text retain ratios}
    \label{fig:cut}
\end{figure}

{
\def\toywidth{0.32\textwidth}
\begin{figure*}[!tbp]
    \centering
    \begin{subfigure}[b]{\toywidth}
        \centering
        \includegraphics[width=\textwidth]{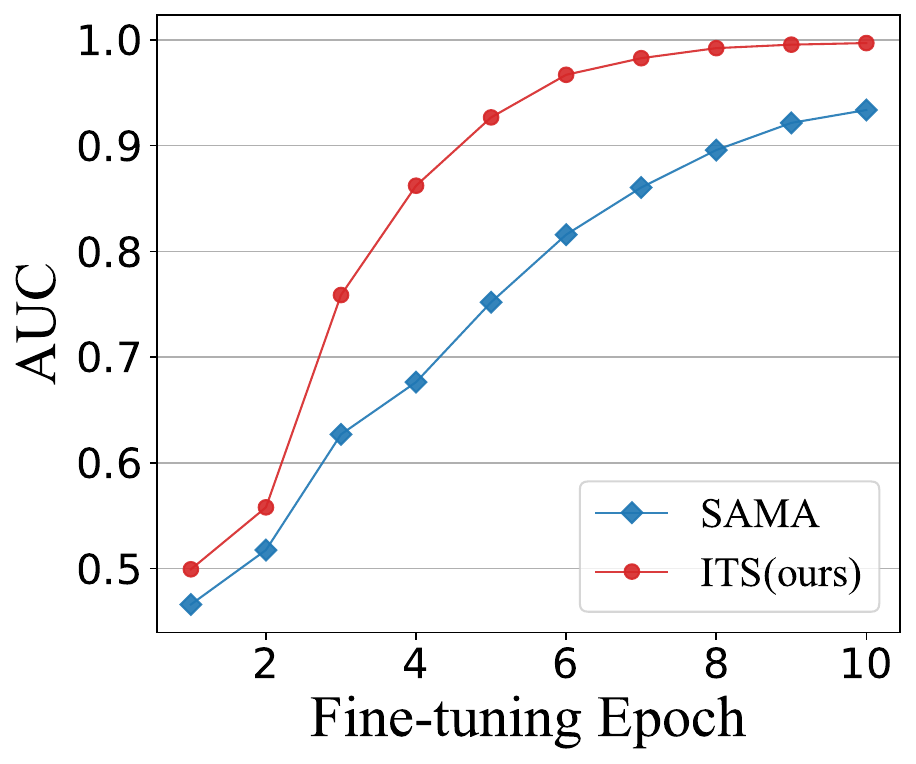}
        \caption{}
        \label{fig:epoch}
    \end{subfigure}
    \begin{subfigure}[b]{\toywidth}
        \centering
        \includegraphics[width=\textwidth]{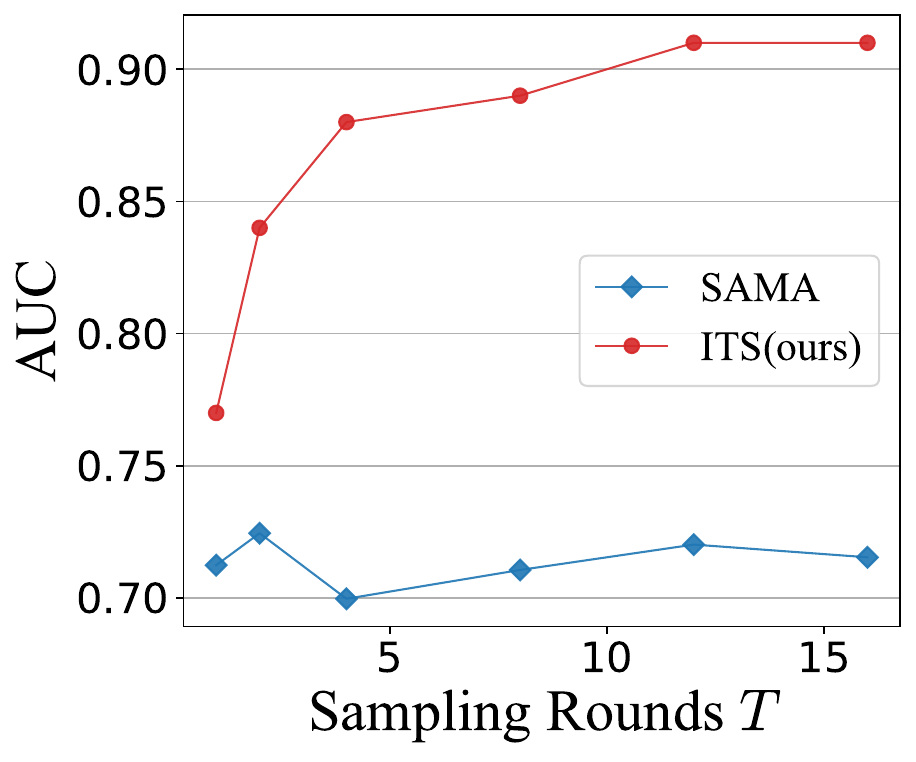}
        \caption{}
        \label{fig:steps}
    \end{subfigure}
    \begin{subfigure}[b]{\toywidth}
        \centering
        \includegraphics[width=\textwidth]{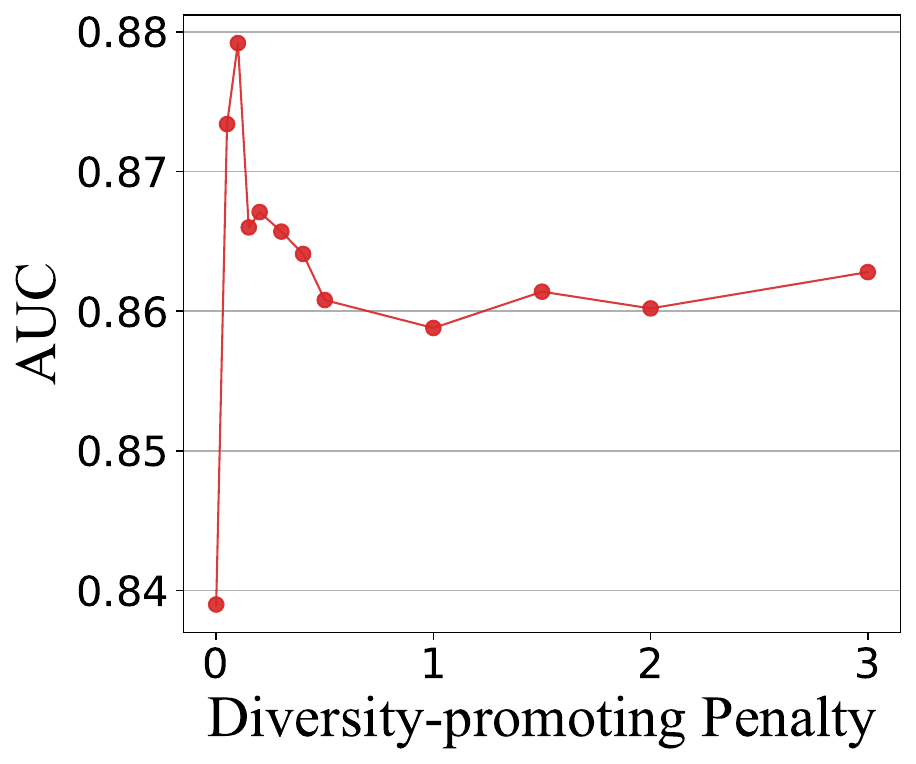}
        \caption{}
        \label{fig:penalty}
    \end{subfigure}
    \caption{Impact of key parameters on detection AUC. (a) Training epochs (b) Sampling rounds (c) Diversity-promoting penalty $\lambda$.}
    \label{fig:motivation}
\end{figure*}
}

\subsection{Robustness Evaluation}

To evaluate detection robustness on content snippets, we truncate test sequences to various retain ratios ($30\%$ to $100\%$). As shown in Fig. \ref{fig:cut}, ITS demonstrates superior data efficiency; even with $30\%$ of the sequence, it maintains a robust lead over SAMA. Notably, the performance gap widens as more context is available, reaching an AUC of 0.88 at full length. This indicates that ITS more effectively aggregates evidence from longer dependencies by mitigating structural bias. See Appendix \ref{appx:robustness_full} for detailed numerical results. Additionally, we explore the robustness of temperature scaling in Appendix \ref{appx:temperature}. 

\subsection{Scaling Analysis}

We examine how detection performance scales with the degree of fine-tuning epochs and the computational budget. The detailed numerical results are presented in Appendix \ref{appx:scaling analysis}.

\paragraph{Training Epochs.}
As shown in Fig. \ref{fig:epoch}, detectability increases with the number of training epochs as the model further internalizes the training distribution. Notably, ITS exhibits a significantly steeper improvement trajectory than SAMA in early epochs, indicating its superior sensitivity to initial stages of memorization.


\paragraph{Sampling Rounds.} 
Increasing the query budget $T$ enhances the precision of likelihood estimation. Fig. \ref{fig:steps} shows that while SAMA's performance fluctuates within a narrow range due to approximation bias, ITS scales consistently with $T$, achieving an AUC of 0.91 at $T=16$. This confirms that ITS more effectively leverages additional sampling to refine the detection signal.

\subsection{Ablation Study}


\paragraph{Impact of Dependency Levels.}
We investigate how the dependency level of sampled token subsets affects. 
Implementation details are provided in Appendix~\ref{appx:impl pcmi}. The low-dependency setting corresponds to our method.  The medium-dependency setting adopts random sampling, whereas the high-dependency setting favors subsets with stronger inter-token dependency. As shown in Table~\ref{tab:pcmi}, lower dependency among selected tokens leads to stronger detection performance. 

\input{table/main/pcmi}

\paragraph{Diversity-promoting Penalty.}



We evaluate the impact of the penalty $\lambda$, designed to reduce sampling redundancy and enhance query efficiency. As shown in Fig.~\ref{fig:penalty}, any $\lambda > 0$ improves AUC from 0.84 to a peak of 0.88. This confirms that penalty effectively surfaces more discriminative signals within a limited query budget.

%% file: table/main/main_llada.tex
{
\def\tablewidth{\linewidth}

\begin{table*}[!ht]
    \setlength{\tabcolsep}{5pt}
    \normalsize
    \centering
    \resizebox{\tablewidth}{!}{
    \begin{tabular}{lccccccccc}
        \toprule
 \multirow{2}{*}{{Method}} 
 & \multicolumn{3}{c}{ArXiv} &\multicolumn{3}{c}{Pile CC} &\multicolumn{3}{c}{AG News} \\
 \cmidrule(lr){2-4}\cmidrule(lr){5-7}\cmidrule(lr){8-10} & 
 AUC & T@10\% & T@1\% & AUC & T@10\% & T@1\% & AUC & T@10\% & T@1\%  \\ \midrule
%
Loss & $0.52_{\pm0.01}$ & $0.13_{\pm0.01}$ & $0.01_{\pm0.01}$ & $0.52_{\pm0.01}$ & $0.13_{\pm0.01}$ & $0.02_{\pm0.01}$ & $0.52_{\pm0.00}$ & $0.10_{\pm0.01}$ & $0.01_{\pm0.01}$ \\
Zlib & $0.49_{\pm0.01}$ & $0.11_{\pm0.02}$ & $0.01_{\pm0.01}$ & $0.50_{\pm0.01}$ & $0.11_{\pm0.02}$ & $0.01_{\pm0.00}$ & $0.53_{\pm0.01}$ & $0.12_{\pm0.01}$ & $\mathbf{0.02}_{\pm0.01}$ \\
Min-k\% & $0.51_{\pm0.01}$ & $0.13_{\pm0.01}$ & $0.01_{\pm0.01}$ & $0.52_{\pm0.01}$ & $0.13_{\pm0.01}$ & $0.02_{\pm0.01}$ & $0.52_{\pm0.00}$ & $0.10_{\pm0.01}$ & $0.01_{\pm0.01}$ \\
Min-k\%++ & $0.49_{\pm0.01}$ & $0.08_{\pm0.01}$ & $0.01_{\pm0.00}$ & $0.50_{\pm0.02}$ & $0.08_{\pm0.02}$ & $0.01_{\pm0.00}$ & $0.51_{\pm0.01}$ & $0.11_{\pm0.01}$ & $0.01_{\pm0.00}$ \\
DF-MIA & $0.52_{\pm0.01}$ & $0.14_{\pm0.01}$ & $0.01_{\pm0.00}$ & $0.54_{\pm0.02}$ & $0.11_{\pm0.01}$ & $0.01_{\pm0.00}$ & $0.61_{\pm0.02}$ & $0.18_{\pm0.02}$ & $0.01_{\pm0.00}$ \\
Ratio & $0.52_{\pm0.01}$ & $0.11_{\pm0.01}$ & $0.01_{\pm0.01}$ & $0.54_{\pm0.01}$ & $0.13_{\pm0.02}$ & $0.02_{\pm0.00}$ & $0.55_{\pm0.01}$ & $0.11_{\pm0.02}$ & $0.01_{\pm0.00}$ \\
SAMA & $0.70_{\pm0.01}$ & $0.34_{\pm0.03}$ & $0.06_{\pm0.01}$ & $0.92_{\pm0.01}$ & $0.78_{\pm0.02}$ & $0.40_{\pm0.04}$ & $0.67_{\pm0.01}$ & $0.25_{\pm0.02}$ & $0.00_{\pm0.00}$ \\
ITS (Ours) & $\mathbf{0.88}_{\pm0.01}$ & $\mathbf{0.64}_{\pm0.02}$ & $\mathbf{0.26}_{\pm0.06}$ & $\mathbf{0.99}_{\pm0.00}$ & $\mathbf{0.97}_{\pm0.00}$ & $\mathbf{0.88}_{\pm0.02}$ & $\mathbf{0.77}_{\pm0.01}$ & $\mathbf{0.37}_{\pm0.01}$ & ${0.01}_{\pm0.00}$ \\
\bottomrule
    \end{tabular}
    }
    
    \resizebox{\tablewidth}{!}{
    \begin{tabular}{lccccccccc}
        \toprule
 \multirow{2}{*}{{Method}} 
 & \multicolumn{3}{c}{PubMed Central} &\multicolumn{3}{c}{Wikipedia (en)} &\multicolumn{3}{c}{Github} \\
 \cmidrule(lr){2-4}\cmidrule(lr){5-7}\cmidrule(lr){8-10} & 
 AUC & T@10\% & T@1\% & AUC & T@10\% & T@1\% & AUC & T@10\% & T@1\%  \\ \midrule
%
Loss & $0.51_{\pm0.01}$ & $0.14_{\pm0.01}$ & $0.01_{\pm0.01}$ & $0.51_{\pm0.01}$ & $0.13_{\pm0.01}$ & $0.01_{\pm0.00}$ & $0.59_{\pm0.01}$ & $0.22_{\pm0.02}$ & $0.04_{\pm0.01}$ \\
Zlib & $0.48_{\pm0.01}$ & $0.09_{\pm0.02}$ & $0.00_{\pm0.00}$ & $0.49_{\pm0.01}$ & $0.09_{\pm0.01}$ & $0.01_{\pm0.00}$ & $0.58_{\pm0.01}$ & $0.22_{\pm0.02}$ & $0.05_{\pm0.01}$ \\
Min-k\% & $0.51_{\pm0.01}$ & $0.14_{\pm0.01}$ & $0.01_{\pm0.01}$ & $0.51_{\pm0.01}$ & $0.13_{\pm0.01}$ & $0.01_{\pm0.00}$ & $0.59_{\pm0.01}$ & $0.22_{\pm0.02}$ & $0.04_{\pm0.01}$ \\
Min-k\%++ & $0.49_{\pm0.01}$ & $0.10_{\pm0.01}$ & $0.01_{\pm0.01}$ & $0.51_{\pm0.01}$ & $0.12_{\pm0.02}$ & $0.01_{\pm0.00}$ & $0.52_{\pm0.01}$ & $0.08_{\pm0.02}$ & $0.01_{\pm0.00}$ \\
DF-MIA & $0.50_{\pm0.01}$ & $0.11_{\pm0.01}$ & $0.01_{\pm0.00}$ & $0.51_{\pm0.01}$ & $0.10_{\pm0.01}$ & $0.01_{\pm0.01}$ & $0.52_{\pm0.02}$ & $0.09_{\pm0.01}$ & $0.01_{\pm0.00}$ \\
Ratio & $0.52_{\pm0.01}$ & $0.12_{\pm0.01}$ & $0.01_{\pm0.00}$ & $0.50_{\pm0.01}$ & $0.11_{\pm0.01}$ & $0.01_{\pm0.01}$ & $0.58_{\pm0.01}$ & $0.16_{\pm0.01}$ & $0.02_{\pm0.00}$ \\
SAMA & $0.70_{\pm0.02}$ & $0.30_{\pm0.02}$ & $0.04_{\pm0.01}$ & $0.82_{\pm0.01}$ & $0.54_{\pm0.01}$ & $0.20_{\pm0.03}$ & $0.79_{\pm0.01}$ & $0.44_{\pm0.02}$ & $0.08_{\pm0.02}$ \\
ITS (Ours) & $\mathbf{0.77}_{\pm0.01}$ & $\mathbf{0.39}_{\pm0.03}$ & $\mathbf{0.08}_{\pm0.04}$ & $\mathbf{0.95}_{\pm0.00}$ & $\mathbf{0.88}_{\pm0.01}$ & $\mathbf{0.55}_{\pm0.03}$ & $\mathbf{0.86}_{\pm0.01}$ & $\mathbf{0.59}_{\pm0.03}$ & $\mathbf{0.12}_{\pm0.04}$ \\
\bottomrule
    \end{tabular}
    }
\caption{Detection Performance on LLaDA.}
\label{tab:main performance}
\end{table*}

}


%% file: table/main/pcmi.tex
{
\def\tablewidth{\linewidth}

\begin{table}[!ht]
    \setlength{\tabcolsep}{5pt}
    \normalsize
    \centering
    \resizebox{.9\tablewidth}{!}{
    \begin{tabular}{cccc}
        \toprule
Dependency level
 & AUC & T@10\% & T@1\% \\ \midrule
High & 0.57 & 0.16 & 0.02 \\
Medium & 0.81 & 0.49 & 0.13 \\
Low (Ours) & \textbf{0.87} & \textbf{0.64} & \textbf{0.28} \\
\bottomrule
    \end{tabular}
    }
\caption{Ablation study on dependency levels.}
    \label{tab:pcmi}
\end{table}
}

%% file: section/7_conclusion.tex
\section{Conclusion}


This paper addresses the challenge of training data detection in dLLMs by identifying a structural error in probability estimation. To mitigate this, we propose Independent Token Sampling (ITS), which leverages internal attention structures to minimize token dependence. Our results demonstrate that ITS significantly improves both accuracy and query efficiency, providing a principled and scalable solution for verifying data compliance in dLLMs.

%% file: section/8_limitation.tex
\section*{Limitations}

Despite its effectiveness, ITS has several limitations. (1) Generalization to PEFT: Since ITS identifies independent tokens based on the reference model's attention structure, its efficacy may slightly decrease when auditing models adapted via Attention LoRA. We hypothesize that because LoRA updates are constrained to a narrow parameter subspace, the resulting memorization signals might be more localized or diffused, making them less aligned with the global dependency patterns captured by the base model's Dependency Matrix. (2) Lack of adaptive attack exploration: We have not yet explored the framework's resilience against adaptive fine-tuning attacks, where an adversary might employ "audit-aware" objectives to intentionally mask leakage within highly redundant token clusters. Investigating the robustness of ITS against such adversarial countermeasures is left as a critical direction for future work.

\section*{Ethical Considerations}

This work aims to support responsible training data detection for diffusion large language models. Our method is designed to help determine whether a given candidate sample may have been used during training or fine-tuning, which can benefit privacy protection, copyright compliance, and model accountability.

We acknowledge that training data detection methods may also be misused to probe sensitive models or infer the presence of private records. Our method does not reconstruct training data or extract private content from the model; it only provides a likelihood-based detection score for a given input sample. 


This work uses publicly available open-source datasets and open-source codebases for research purposes. We reviewed and complied with the licenses and terms of use associated with these artifacts. Our use of these artifacts is consistent with their intended research and academic use. No proprietary or privately collected data were used in this study. We will release our implementation code for research and reproducibility purposes, subject to the licenses and usage conditions of the original artifacts.

%% file: appendix/proof.tex
\section{Proof of Theorems}
\label{sec:证明}



The proof of Theorem~\ref{thm:定理1}:
\begin{proof}
\begin{align}
&\ \text{error}\Big[\hat{\mathcal{P}}(x)\Big],\\
=&\ \mathcal{KL}\Big(\mathcal{P}(x)\Big\Vert\hat{\mathcal{P}}(x)\Big),\\
=&\ \mathcal{KL}\Big(\mathcal{P}(x)\Big\Vert\prod_{i=1}^{\lvert x\rvert}\mathcal{P}(x_i)\Big),\\
=&\sum_x\mathcal{P}(x)\log\frac{\mathcal{P}(x)}{\prod_{i=1}^{\lvert x\rvert}\mathcal{P}(x_i)},\\
=&\sum_x\mathcal{P}(x)\log\prod_{i=2}^{\lvert x\rvert}\frac{\mathcal{P}(x_{\le i})}{\mathcal{P}(x_{<i})\mathcal{P}(x_{i})},\\
=&\sum_x\mathcal{P}(x)\sum_{i=2}^{\lvert x\rvert}\log\frac{\mathcal{P}(x_{\le i})}{\mathcal{P}(x_{<i})\mathcal{P}(x_{i})},\\
=&\sum_{i=2}^{\lvert x\rvert}\sum_x\mathcal{P}(x)\log\frac{\mathcal{P}(x_{\le i})}{\mathcal{P}(x_{<i})\mathcal{P}(x_{i})},\\
=&\sum_{i=2}^{\lvert x\rvert}\sum_{x_{\le i}}\mathcal{P}(x_{\le i})\log\frac{\mathcal{P}(x_{\le i})}{\mathcal{P}(x_{<i})\mathcal{P}(x_{i})},\\
=&\sum_{i=2}^{\lvert x\rvert}\mathcal{I}\big(X_{<i};X_{i}\big).
\end{align}
\end{proof}

The proof of Theorem \ref{thm:定理2} is similar to above, replacing all $\mathcal{P}(\cdot)$ with $\mathcal{P}(\cdot|c)$.

%% file: table/main/main_dream.tex
{
\def\tablewidth{.9\linewidth}

\begin{table*}[!ht]
    \setlength{\tabcolsep}{5pt}
    \normalsize
    \centering
    \resizebox{\tablewidth}{!}{
    \begin{tabular}{lccccccccc}
        \toprule
 \multirow{2}{*}{{Method}} 
 & \multicolumn{3}{c}{ArXiv} &\multicolumn{3}{c}{Pile CC} &\multicolumn{3}{c}{AG News} \\
 \cmidrule(lr){2-4}\cmidrule(lr){5-7}\cmidrule(lr){8-10} & 
 AUC & T@10\% & T@1\% & AUC & T@10\% & T@1\% & AUC & T@10\% & T@1\%  \\ \midrule
%
Loss & $0.52_{\pm0.01}$ & $0.13_{\pm0.01}$ & $0.01_{\pm0.00}$ & $0.51_{\pm0.01}$ & $0.11_{\pm0.01}$ & $0.02_{\pm0.01}$ & $0.53_{\pm0.01}$ & $0.11_{\pm0.01}$ & $0.01_{\pm0.00}$ \\
Zlib & $0.49_{\pm0.01}$ & $0.12_{\pm0.01}$ & $0.01_{\pm0.00}$ & $0.48_{\pm0.01}$ & $0.09_{\pm0.02}$ & $0.01_{\pm0.00}$ & $0.54_{\pm0.01}$ & $0.14_{\pm0.01}$ & $0.02_{\pm0.01}$ \\
Min-k\% & $0.52_{\pm0.01}$ & $0.13_{\pm0.01}$ & $0.01_{\pm0.00}$ & $0.51_{\pm0.01}$ & $0.11_{\pm0.01}$ & $0.02_{\pm0.01}$ & $0.53_{\pm0.01}$ & $0.11_{\pm0.01}$ & $0.01_{\pm0.00}$ \\
Min-k\%++ & $0.50_{\pm0.01}$ & $0.08_{\pm0.01}$ & $0.01_{\pm0.00}$ & $0.49_{\pm0.01}$ & $0.07_{\pm0.02}$ & $0.01_{\pm0.00}$ & $0.51_{\pm0.01}$ & $0.11_{\pm0.01}$ & $0.01_{\pm0.00}$ \\
DF-MIA & $0.52_{\pm0.01}$ & $0.10_{\pm0.01}$ & $0.00_{\pm0.00}$ & $0.51_{\pm0.02}$ & $0.08_{\pm0.01}$ & $0.01_{\pm0.00}$ & $0.60_{\pm0.01}$ & $0.19_{\pm0.02}$ & $0.02_{\pm0.01}$ \\
Ratio & $0.52_{\pm0.01}$ & $0.10_{\pm0.01}$ & $0.01_{\pm0.01}$ & $0.52_{\pm0.01}$ & $0.11_{\pm0.02}$ & $0.01_{\pm0.00}$ & $0.56_{\pm0.01}$ & $0.14_{\pm0.01}$ & $0.02_{\pm0.00}$ \\
SAMA & $0.65_{\pm0.01}$ & $0.23_{\pm0.02}$ & $0.05_{\pm0.01}$ & $0.74_{\pm0.01}$ & $0.38_{\pm0.02}$ & $0.10_{\pm0.02}$ & $0.65_{\pm0.01}$ & $0.22_{\pm0.02}$ & $0.00_{\pm0.00}$ \\
ITS (Ours) & $\mathbf{0.77}_{\pm0.01}$ & $\mathbf{0.46}_{\pm0.03}$ & $\mathbf{0.16}_{\pm0.03}$ & $\mathbf{0.87}_{\pm0.01}$ & $\mathbf{0.68}_{\pm0.03}$ & $\mathbf{0.30}_{\pm0.06}$ & $\mathbf{0.74}_{\pm0.01}$ & $\mathbf{0.33}_{\pm0.02}$ & $\mathbf{0.02}_{\pm0.04}$ \\
\bottomrule
    \end{tabular}
    }
    
    \resizebox{\tablewidth}{!}{
    \begin{tabular}{lccccccccc}
        \toprule
 \multirow{2}{*}{{Method}} 
 & \multicolumn{3}{c}{PubMed Central} &\multicolumn{3}{c}{Wikipedia (en)} &\multicolumn{3}{c}{Github} \\
 \cmidrule(lr){2-4}\cmidrule(lr){5-7}\cmidrule(lr){8-10} & 
 AUC & T@10\% & T@1\% & AUC & T@10\% & T@1\% & AUC & T@10\% & T@1\%  \\ \midrule
%
Loss & $0.52_{\pm0.01}$ & $0.13_{\pm0.01}$ & $0.01_{\pm0.00}$ & $0.52_{\pm0.01}$ & $0.14_{\pm0.01}$ & $0.01_{\pm0.00}$ & $0.56_{\pm0.01}$ & $0.18_{\pm0.02}$ & $0.02_{\pm0.01}$ \\
Zlib & $0.48_{\pm0.01}$ & $0.10_{\pm0.01}$ & $0.01_{\pm0.01}$ & $0.49_{\pm0.01}$ & $0.10_{\pm0.01}$ & $0.01_{\pm0.00}$ & $0.55_{\pm0.01}$ & $0.19_{\pm0.02}$ & $0.03_{\pm0.01}$ \\
Min-k\% & $0.52_{\pm0.01}$ & $0.13_{\pm0.01}$ & $0.01_{\pm0.00}$ & $0.52_{\pm0.01}$ & $0.13_{\pm0.01}$ & $0.01_{\pm0.00}$ & $0.56_{\pm0.01}$ & $0.18_{\pm0.02}$ & $0.02_{\pm0.01}$ \\
Min-k\%++ & $0.49_{\pm0.01}$ & $0.10_{\pm0.01}$ & $0.01_{\pm0.00}$ & $0.52_{\pm0.01}$ & $0.12_{\pm0.02}$ & $0.00_{\pm0.00}$ & $0.53_{\pm0.01}$ & $0.09_{\pm0.02}$ & $0.01_{\pm0.00}$ \\
DF-MIA & $0.50_{\pm0.01}$ & $0.11_{\pm0.02}$ & $0.01_{\pm0.00}$ & $0.51_{\pm0.01}$ & $0.11_{\pm0.01}$ & $0.01_{\pm0.00}$ & $0.59_{\pm0.01}$ & $0.16_{\pm0.02}$ & $0.03_{\pm0.01}$ \\
Ratio & $0.53_{\pm0.01}$ & $0.12_{\pm0.01}$ & $0.01_{\pm0.01}$ & $0.51_{\pm0.01}$ & $0.12_{\pm0.02}$ & $0.01_{\pm0.00}$ & $0.50_{\pm0.01}$ & $0.10_{\pm0.01}$ & $0.02_{\pm0.01}$ \\
SAMA & $0.68_{\pm0.01}$ & $0.30_{\pm0.01}$ & $0.08_{\pm0.01}$ & $0.73_{\pm0.01}$ & $0.35_{\pm0.02}$ & $0.08_{\pm0.01}$ & $0.68_{\pm0.01}$ & $0.29_{\pm0.02}$ & $0.05_{\pm0.01}$ \\
ITS (Ours)  & $\mathbf{0.78}_{\pm0.01}$ & $\mathbf{0.46}_{\pm0.03}$ & $\mathbf{0.12}_{\pm0.02}$ & $\mathbf{0.84}_{\pm0.01}$ & $\mathbf{0.59}_{\pm0.02}$ & $\mathbf{0.25}_{\pm0.02}$ & $\mathbf{0.75}_{\pm0.01}$ & $\mathbf{0.43}_{\pm0.03}$ & $\mathbf{0.14}_{\pm0.02}$ \\
\bottomrule
    \end{tabular}
    }
    \caption{Detection Performance on Dream.}
    \label{tab:main performance dream}
\end{table*}

}

%% file: appendix/baseline.tex

%% file: appendix/details.tex
\section{Experimental Details}
\label{appx:exp details}






\subsection{Model Training}
\label{appx:exp details training}


Our experimental suite comprises six datasets: AG News \citep{agnews} and five subsets from the MIMIR benchmark \citep{mimir} (ArXiv, Pile-CC, Wikipedia (en), PubMed Central, and GitHub). To ensure a fair comparison, we adopt the training and test split configurations proposed by \citet{sama}, fine-tuning the model on each full training partition. Unless otherwise specified, the ArXiv dataset serves as the representative corpus for all robustness evaluations, scaling experiments, and ablation studies.

We consider two types of models: LLaDA-8B-Base \citep{llada} and Dream-v0-Base-7B \citep{dream}. The former is trained from scratch, while the latter is adapted from an AR-LLM. For the reference models used in Ratio \citep{ratio}, SAMA \citep{sama}, and ITS (ours), we adopt the corresponding pre-trained versions. Fine-tuning is conducted using the AdamW optimizer, with a learning rate of $5\times10^{-5}$ for full fine-tuning and $5\times10^{-4}$ for LoRA, a weight decay of 0.1, a batch size of 48, and up to 4 epochs. Training is performed in bf16 precision on one NVIDIA A6000 GPU with DeepSpeed Stage 1 optimization.

\subsection{Detection Inference}
\label{appx:exp details inference}

We follow the implementation of SAMA \citep{sama} for all experimental settings, except for the number of progressive sampling rounds $T$. Specifically, SAMA uses $T=16$, whereas we adopt a substantially smaller budget of $T=4$ in our experiments. The masking ratio increases progressively from $\rho_{\min}=5\%$ to $\rho_{\max}=50\%$, and $N=128$ subsets, each containing $M=10$ tokens, are sampled. All baselines are evaluated under the same query budget, \ie, using $T$ queries to the target model for each sample. For our method, the default diversity-promoting penalty $\lambda$ is set to 0.1.

\subsection{Implementation of Ablation Study on Dependency Levels}

\label{appx:impl pcmi}


To examine how the dependency level of sampled token subsets affects detection performance, we conduct an ablation study while keeping the diversity-promoting penalty unchanged. Specifically, Line 7 in Algorithm~\ref{alg:its} is modified as
\begin{equation}
    u_v \leftarrow \alpha h_v + \lambda C_v .
\end{equation}
In the default ITS setting, we use $\alpha=1$, which encourages the selection of tokens with low cumulative dependency. For the medium-dependency setting, token affinity is ignored and tokens are sampled randomly, corresponding to $\alpha=0$. For the high-dependency setting, we set $\alpha=-1$, which encourages the selection of token sets with high mutual affinity.

%% file: appendix/add_exp.tex
\section{Additional Experimental Results}
\label{appx:additional results}

\subsection{Comparison on Dream Model}
\label{appx:dream result}

To further validate the scalability of our method, we evaluate ITS on the Dream 7B model \citep{dream} across six diverse datasets: ArXiv, Pile CC, AG News, PubMed Central, Wikipedia (en), and Github. As detailed in Table \ref{tab:main performance dream}, ITS consistently achieves the best performance across all datasets and metrics. Conventional AR-based methods (e.g., Min-k\%, Ratio, DF-MIA) struggle significantly on dLLMs, with AUCs often hovering near chance level ($\approx 0.52$). Compared to SAMA, ITS provides a substantial boost in both detection accuracy and high-precision regimes. For instance, on the Pile CC dataset, ITS improves the AUROC from 0.74 to 0.87 and triples the TPR@1\% (0.10 $\to$ 0.30).These results underscore that the structural bias mitigated by ITS is a universal challenge across different dLLM scales and data domains. The consistent lead in TPR@1\% highlights ITS's effectiveness in low-false-positive scenarios, which is critical for real-world training data detection.

\input{table/main/lora}

\subsection{Detailed Results for LoRA Experiments}
\label{appx:lora_full}

In this section, we provide the comprehensive numerical breakdown of the LoRA experiments. These experiments evaluate the sensitivity of training data detection when targeting different model components (Attention vs. MLP) across various Low-Rank Adaptation (LoRA) ranks $r \in \{1, 2, 4, 8, 16, 32, 64, 128\}$.

\textbf{Experimental Setup.} For each rank $r$, the model is fine-tuned on the target dataset using LoRA applied exclusively to either the MLP or Attention modules. During detection, ITS and the SAMA baseline are applied to the same modules that were adapted during training. The sampling rounds $T$ and other hyperparameters are kept identical to ensure a fair comparison.

\textbf{Full Numerical Analysis.} Tables \ref{tab:lora attn} and \ref{tab:lora mlp} list the AUROC, TPR@10\%, and TPR@1\% for both methods.
\begin{enumerate}
\item \textbf{Attention Layers (Table \ref{tab:lora attn}):} Detection performance in Attention layers remains relatively low even as rank increases. ITS consistently maintains a lead over SAMA, particularly in TPR metrics at higher ranks, but the overall signal remains sparse, supporting the view that Attention primarily manages contextual routing rather than long-term knowledge storage.
\item \textbf{MLP Layers (Table \ref{tab:lora mlp}):} In contrast, detection signals in MLP layers are remarkably strong. ITS demonstrates exceptional efficiency, achieving an AUC of 0.8394 at $r=2$, a performance level that SAMA does not reach until $r=32$. At $r=128$, ITS achieves a near-perfect AUC of 0.9984 (rounded to 1.00 in the table), confirming that MLPs are the primary repository for internalized training signatures.
\end{enumerate}

\input{table/main/cut}

\subsection{Robustness Evaluation}
\label{appx:robustness_full}

We evaluate detection performance on partial sequences to simulate real-world copyright auditing of short snippets. Sequences are truncated from the end to retain $30\%, 50\%, 75\%,$ and $90\%$ of the original tokens. Table \ref{tab:cut} provides the full AUROC breakdown. The results show that ITS scales more effectively than SAMA as the observation window expands, confirming that our method captures training signals more accurately from increased context.

\input{table/main/epoch_num}

\input{table/main/step}

\subsection{Scaling Analysis}
\label{appx:scaling analysis}

\paragraph{Training Epochs:} Table \ref{tab:epoch} details the AUROC across 10 epochs. The rapid saturation of ITS performance (reaching >0.98 by epoch 7) demonstrates its effectiveness in auditing models that have moderately internalized protected data.

\paragraph{Sampling Rounds:} Table \ref{tab:step} compares the impact of sampling rounds $T$. The results highlight that ITS is not only more accurate but also more sample-efficient, as ITS at $T=2$ already outperforms SAMA at $T=16$.

\input{table/main/penalty}

\subsection{Detailed Results for Penalty $\lambda$}

Detailed Analysis of the Diversity-promoting Penalty. The penalty $\lambda$ is specifically integrated to discourage the redundant selection of overlapping token sets, thereby maximizing the information gain per query. Table~\ref{tab:penalty} demonstrates the efficacy of this strategy. The AUC and T@10\% scores peak between $\lambda=0.05$ and $\lambda=0.15$. While the AUC slightly retracts to 0.86 for larger $\lambda$ values, it remains remarkably stable even as the penalty increases to 3.00. This plateau indicates that while the balance between independence and coverage is sensitive at lower scales, the ITS framework is not overly dependent on meticulous hyperparameter tuning.

\subsection{Temperature Scaling}
\label{appx:temperature}

\begin{figure}
    \centering
    \includegraphics[width=0.8\linewidth]{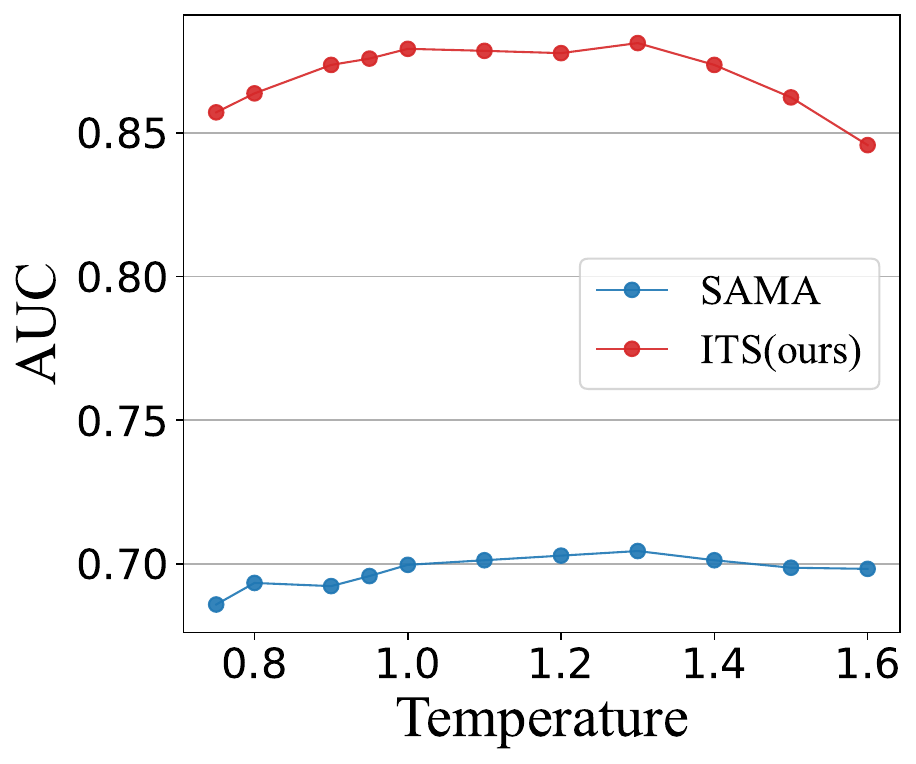}
    \caption{Detection performance under different temperature.}
    \label{fig:temperature}
\end{figure}

Temperature scaling is often used to obfuscate signals by smoothing output distributions with temperature $t$:
\begin{equation}
    p_i = \frac{\exp(z_i/t)}{\sum_j \exp(z_j/t)}.
\end{equation}
However, our results in Table \ref{tab:temperature} and Figure \ref{fig:temperature} show that ITS is remarkably resilient to this perturbation, maintaining a dominant performance lead over SAMA across all tested temperatures $t \in [0.75, 1.60]$. As shown in Table \ref{tab:temperature}, ITS consistently achieves an AUC between 0.85 and 0.88, whereas SAMA plateaus at approximately 0.70. Even under aggressive scaling ($T=1.5$), ITS retains a TPR@1\%FPR of 0.28, which has a four-fold improvement over SAMA’s 0.07. The stability of ITS suggests that by minimizing mutual dependency, our method captures intrinsic memorization signals that are largely invariant to monotonic logit transformations. While temperature scaling increases output entropy, the relative information captured by our independent token sets remains intact. Consequently, temperature scaling fails as a viable defense against ITS, as our method continues to provide a robust and discriminative signal where heuristic baselines falter.

\input{table/main/temperature}

%% file: table/main/lora.tex
{
\def\tablewidth{\linewidth}

\begin{table}[!ht]
    \setlength{\tabcolsep}{5pt}
    \normalsize
    \centering
    \resizebox{\tablewidth}{!}{
    \begin{tabular}{ccccccc}
        \toprule
 \multirow{2}{*}{{Rank}} 
 & \multicolumn{2}{c}{AUC} &\multicolumn{2}{c}{T@10\%} &\multicolumn{2}{c}{T@1\%} \\
 \cmidrule(lr){2-3}\cmidrule(lr){4-5}\cmidrule(lr){6-7} & 
 SAMA & ITS & SAMA & ITS &SAMA & ITS  \\ \midrule
1 & 0.48 & \textbf{0.50} & \textbf{0.10} & 0.09 & 0.01 & {0.01} \\
2 & 0.50 & \textbf{0.52} & 0.09 & \textbf{0.12} & {0.01} & 0.01 \\
4 & 0.51 & \textbf{0.53} & 0.11 & \textbf{0.13} & 0.01 & {0.01} \\
8 & 0.52 & \textbf{0.56} & 0.13 & {0.13} & 0.01 & \textbf{0.02} \\
16 & 0.52 & \textbf{0.58} & 0.10 & \textbf{0.17} & 0.02 & 0.02 \\
32 & 0.56 & \textbf{0.64} & 0.15 & \textbf{0.24} & 0.01 & \textbf{0.03} \\
64 & 0.60 & \textbf{0.70} & 0.21 & \textbf{0.30} & 0.02 & \textbf{0.06} \\
128 & 0.59 & \textbf{0.70} & 0.18 & \textbf{0.31} & 0.04 & \textbf{0.06} \\
\bottomrule
    \end{tabular}
    }
    \caption{Detection performance with different LoRA ranks. The target modules are attention layers.}
    \label{tab:lora attn}
\end{table}

}

{
\def\tablewidth{\linewidth}

\begin{table}[!ht]
    \setlength{\tabcolsep}{5pt}
    \normalsize
    \centering
    \resizebox{\tablewidth}{!}{
    \begin{tabular}{ccccccc}
        \toprule
 \multirow{2}{*}{{Rank}} 
 & \multicolumn{2}{c}{AUC} &\multicolumn{2}{c}{T@10\%} &\multicolumn{2}{c}{T@1\%} \\
 \cmidrule(lr){2-3}\cmidrule(lr){4-5}\cmidrule(lr){6-7} & 
 SAMA & ITS & SAMA & ITS &SAMA & ITS  \\ \midrule
1 & 0.63 & \textbf{0.76} & 0.20 & \textbf{0.40} & 0.03 & \textbf{0.08} \\
2 & 0.69 & \textbf{0.85} & 0.28 & \textbf{0.54} & 0.06 & \textbf{0.23} \\
4 & 0.68 & \textbf{0.86} & 0.27 & \textbf{0.59} & 0.06 & \textbf{0.25} \\
8 & 0.73 & \textbf{0.90} & 0.34 & \textbf{0.68} & 0.10 & \textbf{0.24} \\
16 & 0.76 & \textbf{0.93} & 0.40 & \textbf{0.76} & 0.08 & \textbf{0.43} \\
32 & 0.83 & \textbf{0.97} & 0.52 & \textbf{0.90} & 0.10 & \textbf{0.62} \\
64 & 0.88 & \textbf{0.99} & 0.65 & \textbf{0.96} & 0.23 & \textbf{0.78} \\
128 & 0.93 & \textbf{1.00} & 0.79 & \textbf{0.99} & 0.40 & \textbf{0.93} \\
\bottomrule
    \end{tabular}
    }
\caption{Detection performance with different LoRA ranks. The target modules are MLP layers.}
    \label{tab:lora mlp}
\end{table}

}

%% file: table/main/cut.tex
{
\def\tablewidth{\linewidth}

\begin{table}[!ht]
    \setlength{\tabcolsep}{5pt}
    \normalsize
    \centering
    \resizebox{\tablewidth}{!}{
    \begin{tabular}{ccccccc}
        \toprule
 \multirow{2}{*}{{Retain Ratio}} 
 & \multicolumn{2}{c}{AUC} &\multicolumn{2}{c}{T@10\%} &\multicolumn{2}{c}{T@1\%} \\
 \cmidrule(lr){2-3}\cmidrule(lr){4-5}\cmidrule(lr){6-7} & 
 SAMA & ITS & SAMA & ITS &SAMA & ITS  \\ \midrule
 30\% & 0.59 & \textbf{0.63} & 0.17 & \textbf{0.20} & 0.01 & \textbf{0.02} \\
50\% & 0.64 & \textbf{0.76} & 0.21 & \textbf{0.40} & 0.05 & \textbf{0.07} \\
75\% & 0.67 & \textbf{0.81} & 0.28 & \textbf{0.47} & 0.05 & \textbf{0.15} \\
90\% & 0.67 & \textbf{0.83} & 0.27 & \textbf{0.53} & 0.06 & \textbf{0.22} \\
100\% & 0.70 & \textbf{0.88} & 0.34 & \textbf{0.64} & 0.06 & \textbf{0.26} \\
\bottomrule
    \end{tabular}
}
    \caption{Detection performance with different text retain ratios.}
    \label{tab:cut}
\end{table}

}

%% file: table/main/epoch_num.tex
{
\def\tablewidth{\linewidth}

\begin{table}[!ht]
    \setlength{\tabcolsep}{5pt}
    \normalsize
    \centering
    \resizebox{\tablewidth}{!}{
    \begin{tabular}{ccccccc}
        \toprule
 \multirow{2}{*}{{Epoch}} 
 & \multicolumn{2}{c}{AUC} &\multicolumn{2}{c}{T@10\%} &\multicolumn{2}{c}{T@1\%} \\
 \cmidrule(lr){2-3}\cmidrule(lr){4-5}\cmidrule(lr){6-7} & 
 SAMA & ITS & SAMA & ITS &SAMA & ITS  \\ \midrule

1 & 0.47 & \textbf{0.50} & 0.07 & \textbf{0.09} & 0.01 & \textbf{0.02} \\
2 & 0.52 & \textbf{0.56} & 0.11 & \textbf{0.16} & 0.01 & \textbf{0.02} \\
3 & 0.63 & \textbf{0.76} & 0.21 & \textbf{0.38} & 0.02 & \textbf{0.12} \\
4 & 0.68 & \textbf{0.86} & 0.28 & \textbf{0.61} & 0.05 & \textbf{0.26} \\
5 & 0.75 & \textbf{0.93} & 0.38 & \textbf{0.78} & 0.11 & \textbf{0.38} \\
6 & 0.82 & \textbf{0.97} & 0.51 & \textbf{0.92} & 0.11 & \textbf{0.60} \\
7 & 0.86 & \textbf{0.98} & 0.62 & \textbf{0.95} & 0.18 & \textbf{0.75} \\
8 & 0.90 & \textbf{0.99} & 0.70 & \textbf{0.98} & 0.27 & \textbf{0.88} \\
9 & 0.92 & \textbf{1.00} & 0.78 & \textbf{0.99} & 0.30 & \textbf{0.94} \\
10 & 0.93 & \textbf{1.00} & 0.79 & \textbf{0.99} & 0.43 & \textbf{0.96} \\
\bottomrule
    \end{tabular}
    }
    \caption{Detection performance with different fine-tuning epochs.}
    \label{tab:epoch}
\end{table}

}

%% file: table/main/step.tex
{
\def\tablewidth{\linewidth}

\begin{table}[!ht]
    \setlength{\tabcolsep}{5pt}
    \normalsize
    \centering
    \resizebox{\tablewidth}{!}{
    \begin{tabular}{ccccccc}
        \toprule
 \multirow{2}{*}{{$T$}} 
 & \multicolumn{2}{c}{AUC} &\multicolumn{2}{c}{T@10\%} &\multicolumn{2}{c}{T@1\%} \\
  \cmidrule(lr){2-3}\cmidrule(lr){4-5}\cmidrule(lr){6-7} & 
 SAMA & ITS & SAMA & ITS &SAMA & ITS  \\ \midrule
1 & 0.71 & \textbf{0.77} & 0.34 & \textbf{0.41} & 0.07 & \textbf{0.09} \\
2 & 0.72 & \textbf{0.84} & 0.34 & \textbf{0.58} & 0.07 & \textbf{0.18} \\
4 & 0.70 & \textbf{0.88} & 0.34 & \textbf{0.64} & 0.06 & \textbf{0.26} \\
8 & 0.71 & \textbf{0.89} & 0.33 & \textbf{0.66} & 0.08 & \textbf{0.33} \\
12 & 0.72 & \textbf{0.91} & 0.36 & \textbf{0.75} & 0.06 & \textbf{0.33} \\
16 & 0.72 & \textbf{0.91} & 0.31 & \textbf{0.74} & 0.07 & \textbf{0.43} \\
\bottomrule
    \end{tabular}
    }
\caption{Detection performance with different sampling rounds $T$.}
    \label{tab:step}
\end{table}

}

%% file: table/main/penalty.tex
{
\def\tablewidth{.8\linewidth}

\begin{table}[!ht]
    \setlength{\tabcolsep}{5pt}
    \normalsize
    \centering
    \resizebox{\tablewidth}{!}{
    \begin{tabular}{cccc}
        \toprule
{{$\lambda$}} 
 & {AUC} &{T@10\%} &{T@1\%}  \\ \midrule
0.00 & 0.84 & 0.59 & 0.17 \\
0.05 & 0.87 & 0.63 & 0.24 \\
0.10 & 0.88 & 0.64 & 0.26 \\
0.15 & 0.87 & 0.61 & 0.25 \\
0.20 & 0.87 & 0.63 & 0.25 \\
0.30 & 0.87 & 0.64 & 0.25 \\
0.40 & 0.86 & 0.62 & 0.28 \\
0.50 & 0.86 & 0.61 & 0.24 \\
1.00 & 0.86 & 0.60 & 0.27 \\
1.50 & 0.86 & 0.60 & 0.28 \\
2.00 & 0.86 & 0.60 & 0.29 \\
3.00 & 0.86 & 0.62 & 0.27 \\
\bottomrule
    \end{tabular}
    }
\caption{Detection performance with different diversity-promoting penalty $\lambda$.}
    \label{tab:penalty}
\end{table}

}

%% file: table/main/temperature.tex
{
\def\tablewidth{\linewidth}

\begin{table}[!ht]
    \setlength{\tabcolsep}{5pt}
    \normalsize
    \centering
    \resizebox{\tablewidth}{!}{
    \begin{tabular}{ccccccc}
        \toprule
 \multirow{2}{*}{{$t$}} 
 & \multicolumn{2}{c}{AUC} &\multicolumn{2}{c}{T@10\%} &\multicolumn{2}{c}{T@1\%} \\
  \cmidrule(lr){2-3}\cmidrule(lr){4-5}\cmidrule(lr){6-7} & 
 SAMA & ITS & SAMA & ITS &SAMA & ITS  \\ \midrule
0.75 & 0.69 & \textbf{0.86} & 0.29 & \textbf{0.62} & 0.06 & \textbf{0.26} \\
0.80 & 0.69 & \textbf{0.86} & 0.31 & \textbf{0.60} & 0.06 & \textbf{0.23} \\
0.90 & 0.69 & \textbf{0.87} & 0.30 & \textbf{0.64} & 0.08 & \textbf{0.26} \\
0.95 & 0.70 & \textbf{0.88} & 0.30 & \textbf{0.64} & 0.07 & \textbf{0.26} \\
1.00 & 0.70 & \textbf{0.88} & 0.34 & \textbf{0.64} & 0.06 & \textbf{0.26} \\
1.10 & 0.70 & \textbf{0.88} & 0.31 & \textbf{0.64} & 0.07 & \textbf{0.26} \\
1.20 & 0.70 & \textbf{0.88} & 0.30 & \textbf{0.67} & 0.08 & \textbf{0.27} \\
1.30 & 0.70 & \textbf{0.88} & 0.30 & \textbf{0.66} & 0.07 & \textbf{0.28} \\
1.40 & 0.70 & \textbf{0.87} & 0.30 & \textbf{0.63} & 0.07 & \textbf{0.22} \\
1.50 & 0.70 & \textbf{0.86} & 0.30 & \textbf{0.62} & 0.07 & \textbf{0.28} \\
1.60 & 0.70 & \textbf{0.85} & 0.31 & \textbf{0.58} & 0.05 & \textbf{0.23} \\
\bottomrule
    \end{tabular}
    }
    \caption{Detection performance with different temperature.}
    \label{tab:temperature}
\end{table}

}

%% file: appendix/ai.tex
\section{Use of AI Assistants}

AI-assisted tools were used solely for language polishing and improving the clarity, grammar, and readability of the manuscript. The AI assistant was not used to generate scientific ideas, design the study, perform data analysis, interpret results, or draw conclusions. All content, arguments, experimental results, and conclusions were developed and verified by the authors, who take full responsibility for the final manuscript.